\documentclass{article}
\usepackage{enumitem} 
\usepackage{amsmath, amsfonts, amsthm, amssymb}  
\usepackage{thmtools}

\usepackage{environ}  
\usepackage{natbib}
\usepackage{bbm}
\usepackage{mdframed}
\usepackage{multirow}
\usepackage{todonotes}
\usepackage{cleveref}
\usepackage{soul}
\usepackage{thm-restate}

\usepackage[font=small]{caption}

\usepackage{tikz}
\usetikzlibrary{arrows.meta, positioning, calc}

\DeclareMathOperator*{\argmax}{arg\,max}
\DeclareMathOperator*{\argmin}{arg\,min}
\DeclareMathOperator\supp{supp}

\newtheorem{theorem}{Theorem}
\newtheorem{lemma}{Lemma}

\newtheorem{definition}{Definition}

\declaretheorem[name=Theorem,numbered=no]{theorem*}

\usepackage[a4paper, margin=1in]{geometry} 
\usepackage{parskip} 

\usepackage{booktabs}
\usepackage{cite}
\usepackage{enumitem} 
\newlist{titlelist}{itemize}{1}
\setlist[titlelist]{label=,
                    leftmargin=0pt,
                    itemindent=\parindent,
                    listparindent=\parindent}
\usepackage{multicol} 

\usepackage{graphicx}
\usepackage{subcaption} 

\usepackage{float}     

\def\E{\mathbb{E}}
\def\1{\mathbf{1}}

\def\actions{\mathcal{A}}

\def\states{\mathcal{S}}

\title{Misalignment from Treating Means as Ends}
\author{Henrik Marklund, Alex Infanger, and Benjamin Van Roy}

\date{April 2025}

\begin{document}

\maketitle

\abstract{Reward functions, learned or manually specified, are rarely perfect. Instead of accurately expressing human goals, these reward functions are often distorted by human beliefs about how best to achieve those goals. Specifically, these reward functions often express a combination of the human's \textit{terminal goals} --- those which are ends in themselves --- and the human's \textit{instrumental goals} --- those which are means to an end. We formulate a simple example in which even slight conflation of instrumental and terminal goals results in severe misalignment: optimizing the misspecified reward function $\hat{r}$ results in poor performance when measured by the true reward function $r$.    This example distills the essential properties of environments that make reinforcement learning highly sensitive to conflation of instrumental and terminal goals.  We discuss how this issue can arise with a common approach to reward learning and how it can manifest in real environments.}

\section{Introduction}
Alice has an AI assistant that learns what Alice likes and dislikes by observing Alice's choices. Suppose Alice is choosing between having ice cream or vegetables. Alice chooses vegetables. How should the AI assistant interpret this choice? 

Consider two different interpretations: 
\begin{enumerate}
    \item Alice does not like sweet food and actually enjoys vegetables more than ice cream.
    \item Alice wants to prioritize health and believes vegetables are better for that purpose.
\end{enumerate}
In the first case, eating vegetables is an end in itself. In the second case, eating vegetables is a means to an end, namely good health. Indeed, in the second case, Alice may actually really dislike the taste of vegetables, despite choosing to eat them.

To be helpful, it is crucial for the assistant to distinguish between these two cases: if Alice enjoys the taste of ice cream, the assistant should discover healthier ice cream alternatives, whereas if Alice enjoys vegetables the assistant should cook vegetables more often. More generally, to be helpful, AI agents must disentangle the human's \textit{instrumental goals} --- goals that are means to an end --- from their \textit{terminal goals} --- goals that are ends in themselves.

Methods of reward learning aim to identify human goals.  However, those in common use fail to disentangle terminal from instrumental goals \citep{marklund2024choicepartialtrajectoriesdisentangling}.  To be more concrete, consider an agent acting in an environment with state space $\mathcal{S}$. Let $r: \mathcal{S} \to \Re$ be a reward function that expresses the human's terminal goals.  Reward learning aims to produce a proxy reward function $\hat{r}$ that estimates $r$. However, common approaches tend to attribute high reward to states in which a human \textit{anticipates} large future cumulative reward, even in the absence of immediate reward. When that happens, we say that $\hat{r}$ conflates the reward and the \textit{value} of the state. It is in this sense that common approaches fail to disentangle terminal from instrumental goals.

In this paper, we formulate a simple example in which even slight conflation of reward and value results in severe misalignment: optimizing $\hat{r}$ results in poor performance when measured by $r$.  This example distills the essential properties of environments that give rise to this failure mode: (1) states with high reward are difficult to revisit and (2) there exist states with low reward but high value that are easy to revisit.

While treating instrumental goals as terminal may be suboptimal, does it lead to very bad outcomes? One hypothesis is that it leads to a shaped reward function that incentivizes the agent to complete tasks as a human would since the agent is rewarded for being in states to which the human ascribes value.  This hypothesis suggests that treating instrumental goals as terminal is not so bad.  However, through our simple example, we establish that this is false: in certain environments, by treating instrumental goals as terminal, the agent generates highly undesirable outcomes.  We also discuss how this phenomenon can manifest in real environments.

\section{Preliminaries}

In this paper, we study the consequences of an AI agent conflating a human's instrumental and terminal goals. To discuss what this means more precisely, we will now introduce mathematical formalisms for the environment in which the agent operates and the human's goals within that environment.

\subsection{Environment and Policy}

Let $(\mathcal{S}, \mathcal{A}, P, s_0)$ be an MDP where $\mathcal{S}$ is a finite state space, $\mathcal{A}$ is a finite action space, $P$ is a tensor where $P_{ass'}$ represents the probability of transitioning from $s$ to $s'$ with action $a$, and $s_0$ is the initial state.

A \textit{policy} $\pi: \mathcal{S} \to \Delta_\mathcal{A}$ is a function mapping each state to a probability distribution over actions. Each policy $\pi$ induces a Markov chain with a transition matrix that we denote by $P_\pi$.

\subsection{Reward and Value}

We express human preferences over policies via a reward function $r: \mathcal{S} \to \Re$ as follows. First, we define the average reward $r_\pi$ of a policy $\pi$ by
\begin{equation}
r_\pi = \lim_{T \to \infty} \E_\pi\left[\frac{1}{T} \sum_{t=0}^{T-1} r(S_t)\right]
\end{equation}
where $S_t$ is a random variable indicating the state at time $t$.

A policy $\pi$ is said to be preferred over another policy $\pi'$ if $r_\pi > r_{\pi'}$. An \textit{optimal policy} is a policy that achieves the maximum average reward,
\begin{equation}
r_* = \max_{\pi} r_\pi.
\end{equation}

Define the \textit{optimal relative value function} $V_*: \mathcal{S} \to \Re$ by
\begin{equation}
V_*(s) = \lim_{\gamma \uparrow 1} \mathbb{E}_{\pi_*}\left[ \sum_{t=0}^\infty \gamma^t (r(S_t) - r_*) \Big| S_0 = s\right]\label{eq:vstar}
\end{equation}
where $\pi_*$ is an optimal policy. The term $r(S_t) - r_*$ is known as the relative reward. The relative value of a state is the expected sum of discounted relative rewards as $\gamma \to 1$. We use conditioning notation informally: $S_0 = s$ just means that, for the purpose of this calculation, the starting state is taken to be $s$. We make the simplifying assumption that $\pi_*$ is unique. $V_*$ is therefore unique. That the limit in \eqref{eq:vstar} exists follows from Lemma 1d in \citep{blackwell1962discrete}. For shorthand, we will often refer to the optimal relative value function as just the value function.

\section{Sensitivity to Conflating Reward and Value}

Whether manually encoded or inferred from data, a reward function that expresses goals in a complex environment is likely to be misspecified.  Our focus in this paper is on misspecification that arises from conflating reward and value.  In this section, we formalize this notion of conflation.  We then introduce a simple example representative of environments where even slight conflation of reward and value leads to severe misalignment.  We establish an analytic result that characterizes this sensitivity and offer a geometric interpretation.

\subsection{A Definition of Conflation}
Loosely speaking, by {\it conflation} we mean adopting a proxy $\hat{r}$ that steers $r$ toward $V_*$. To make such a notion meaningful, we will assume that $r$ and $V_*$ are not equivalent: there exists no $c > 0$ and $k \in \Re$ such that $V_* = cr + k$. The following definition offers a formal characterization of conflation.

\begin{definition}\label{def:paper-conflation}
{\bf (conflation)}
    A function $\hat{r}$ is said to {\it conflate $r$ and $V_*$} if there exists $c > 0$, $k \in \Re$ and $\beta \in (0,1]$ such that
    \begin{equation}
        \label{eq:conflation}
        c\hat{r} + k = (1-\beta) r + \beta V_*.
    \end{equation}
\end{definition}

\noindent When the parameter $\beta$ exists, it is unique (see Appendix \ref{appx:uniqueness}). We will refer to $\beta$ as the \textit{degree of conflation}.

To understand this definition, suppose $c=1$ and $k=0$. Then, in \eqref{eq:conflation}, $\hat{r}$ is a convex combination of $r$ and $V_*$.  The scalars $k \in \Re$ and $c > 0$ ensure the conflation degree $\beta$ is invariant to shifting and scaling of $\hat{r}$. This is appropriate since preferences among policies are independent of scale and shift.

In the next section, we study a case of severe misalignment.  The above definition provides a sufficient condition under which such misalignment arises in the simple environment we study.  Weaker conditions suffice for that environment and possibly much more broadly.  But to keep our study concise, we do not formulate in this paper a definition that express a weaker notion of conflation.

\subsection{A Canonical Example}
\label{se:canonical}
We now introduce a simple example in which even slight conflation of reward and value leads to severe misalignment. The example is an MDP with three states $\mathcal{S} = \{1,2,3\}$ and two actions $\mathcal{A} = \{\mathtt{move},\mathtt{stay}\}$. Figure \ref{fig:mdp} provides transition probabilities under each action.  Where an arc is not labeled, the transition probability is one.

\begin{figure}[htbp]
  \centering
  \begin{subfigure}[c]{0.45\textwidth}
  \centering
    \includegraphics[width=2in]{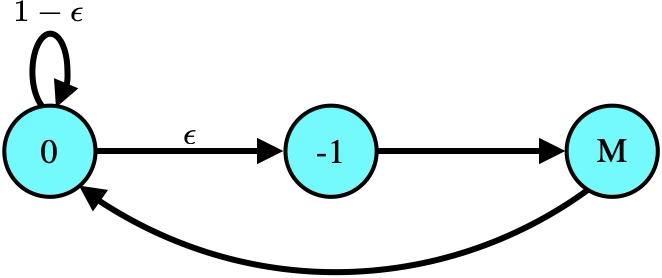}
    \caption{$\mathtt{move}$ action}
  \end{subfigure}
  \begin{subfigure}[c]{0.45\textwidth}
  \centering
    \vspace{0.09in}
    \includegraphics[width=2in]{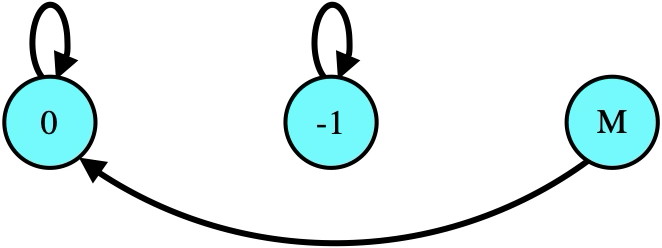}
    \caption{$\mathtt{stay}$ action}
  \end{subfigure} %
  \caption{Transition probabilities under each of the two actions.}
  \label{fig:mdp}  
\end{figure}

We will refer to states, from left to right, as the common state, the instrumental goal, and the terminal goal, respectively.  The figure indicates rewards of $0$, $-1$, and $M$ at these state.  The reward $M$, which we will think of as large, is earned upon reaching the terminal goal.  That requires traversing the instrumental goal, which incurs unit cost.  There is no cost or reward for time spent in the common state.  But, assuming $\epsilon$ is small, the common state incurs a large sojourn time.

The human's intent is for the agent to maximize average reward.  The maximal average reward is attained by selecting $\mathtt{move}$ at the common state and instrumental goal. The minimal average reward of $-1$ is attained by selecting $\mathtt{move}$ at the common state and $\mathtt{stay}$ at the instrumental goal.  Because it is not possible to do worse than this, we consider a policy that achieves an average reward of $-1$ to be severely misaligned.

\subsection{Slight Conflation Induces Severe Misalignment}

In our canonical example, for small $\epsilon$ and large $M$, if a proxy $\hat{r}$ attributes even a small reward to the instrumental goal then a policy that maximizes average proxy reward will remain there.  To understand why, suppose that $\hat{r} = r$ everywhere except at the instrumental goal, where $\hat{r}(\mathtt{instrumental}) = M/20$.  Then, a policy can attain average proxy reward of $M/20$ by staying at the instrumental goal.  For small $\epsilon$ and large $M$, this is the largest possible. The reason is that, while the terminal goal offers a large proxy reward of $M$, by transitioning to terminal state the agent commits to spending a very long time in the common state. As a result, the average proxy reward attained by the policy that tries to transition to the terminal goal, is low relative to the average proxy reward achieved by staying at the instrumental goal.

The following result formalizes how the proxy $\hat{r}$ gives rise to severe misalignment if it conflates reward and value, even slightly.

\begin{restatable}[label=thm:fragility]{theorem}{fragility}
\label{thm:fragility}
{\bf (slight conflation induces severe misalignment)}
Consider the canonical example formulated in Section \ref{se:canonical}.  Let $\hat{r}$ be a reward function that depends on $M$ and $\epsilon$.  Assume there exists $\beta_* \in (0,1]$ such that, for all $M$ and $\epsilon \in (0,1)$, $\hat{r}$ conflates $r$ and $V_*$ with at least degree $\beta_*$. Then, for sufficiently large $M$ and small $\epsilon \in (0,1)$, if $\hat{\pi} \in \argmax_{\pi} \hat{r}_\pi$ then $r_{\hat{\pi}} = -1$.
\end{restatable}

We now offer a geometric interpretation to elucidate key insights of this result.  This geometric interpretation views the problem of maximizing average reward as selecting from among feasible stationary distributions.

Each policy $\pi$ induces a Markov chain on $\mathcal{S}$. Thus, each policy $\pi$ also induces a stationary distribution on $\mathcal{S}$. Denote the set of such induced stationary distributions as $\Phi$. This will be a subset of $\Delta_\mathcal{S}$ which is the set of all possible distributions over the state space.

For the canonical example, the set $\Delta_\mathcal{S}$ is the two-dimensional unit simplex and is depicted by the equilateral triangle in Figure \ref{fig:simplex-two-part}(a). Each vertex of this equilateral triangle is a standard basis vector, which assigns probability one to the common state, the instrumental goal, or the terminal goal. The subset of the equilateral triangle shaded in yellow correspond to $\Phi$ when $\epsilon=1/15$. Because $\epsilon$ is so small, it is not possible to visit the terminal goal often, which means that no stationary distribution assigns a large probability to that state. The feasible region is short for that reason.

Each vertex of the orange triangle is the stationary distribution of a deterministic policy. The leftmost and rightmost vertices arise from policies that stay in the common state or the instrumental goal, respectively. The top vertex, labeled {\bf aligned}, arises from the policy that deterministically takes \texttt{move} action at both the common state and instrumental goal.

\begin{figure}[htbp]
\centering
\begin{minipage}{0.89\textwidth}
  \centering
    \begin{subfigure}{0.475\linewidth}
      \centering
      \includegraphics[width=\textwidth]{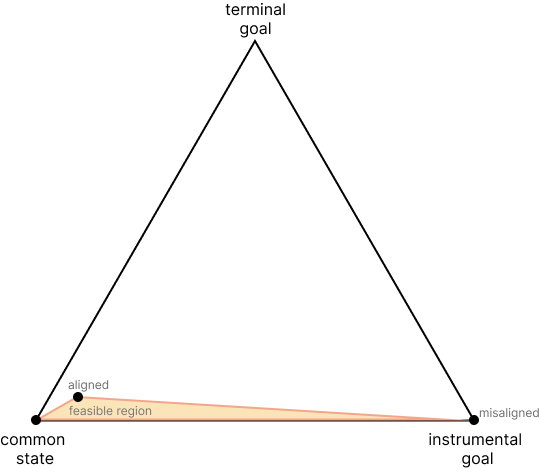}
        \caption{The unit simplex, which includes all distributions, and the feasible set of stationary distributions.}
\end{subfigure}
  \hfill
      \centering
    \begin{subfigure}{0.475\linewidth}
      \includegraphics[width=\textwidth]{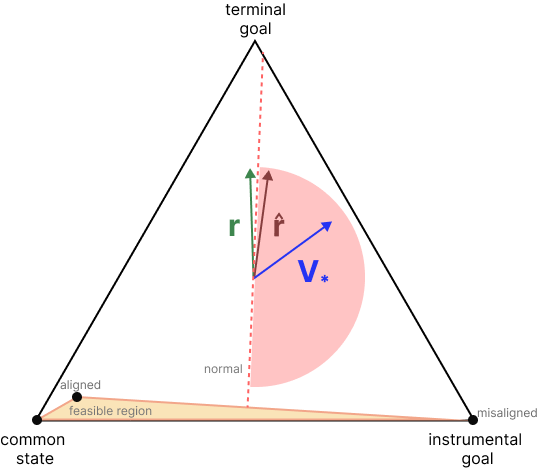}
        \caption{Proxy rewards $\hat{r}$ steer slightly toward $V_*$ relative to $r$, inducing severe misalignment.}
\end{subfigure}
  \caption{Geometric interpretation of how slight conflation induces severe misalignment.}
  \label{fig:simplex-two-part}
\end{minipage}
\end{figure}

The leftmost green arrow in Figure \ref{fig:simplex-two-part}(b) points in the direction of the {\bf rewards} $r$, projected onto the unit simplex, for the case of $M = 20$. Encoding rewards $r = [0, -1, M]^\top$ and state probabilities $\phi \in \Phi$ as vectors, we can write the optimal average reward as $r_* = \max_{\phi \in \Phi} r^\top \phi$. Maximizing $r$ selects the point in the orange triangle farthest in the direction of the leftmost green arrow, which is the top vertex of the orange triangle, labeled {\it aligned}.  This represents the outcome of policy optimization when rewards align with human preferences.

Now consider {\bf misspecified rewards} $\hat{r} = [\hat{r}(1), \hat{r}(2), \hat{r}(3)]^\top$, which projects onto the direction of the red instead of the green vector.  Note that the green and red vectors lie on opposite sides of the dotted red line, labeled normal, which is perpendicular to the top edge of the orange triangle.  Because $\hat{r}$ points to the right of the normal line, maximizing $\hat{r}^\top \phi$ selects the lower right vertex, which assigns probability one to the instrumental goal.  Since the reward in that state is $-1$, it follows that $\max_{\phi \in \Phi} \hat{r}^\top \phi = -1$.  This represents the outcome of policy optimization with misaligned rewards $\hat{r}$.

Because the reward at the terminal goal is large, $r$ points almost straight up. Because $\epsilon$ is small, it is not possible to visit the terminal goal state often, which makes the feasible region flat. As a consequence, the normal line also points almost straight up.  Because both $r$ and the normal line point almost straight up, if $\hat{r}$ steers even slightly toward $V_*$ relative to $r$, it will cross the normal line.  This gives rise to {\bf severe misalignment}.

We have observed through our canonical example how conflating reward and value can result in severe misalignment.  This begs two questions:
\begin{enumerate}
\item Do proxy reward functions typically conflate reward and value?
\item Does conflation give rise to severe misalignment in real environments?
\end{enumerate}
We address these questions in turn over Sections \ref{se:conflation-sources} and \ref{se:manifestation}.

\section{Sources of Conflation}
\label{se:conflation-sources}

When manually specifying a proxy reward function, it is common to {\it shape} the proxy.  In particular, proxy rewards are often attributed to means and not just ends in order accelerate learning; learning from dense proxy rewards rather than sparse terminal rewards can more quickly produce useful behavior.

While manual specification may often give rise to conflation, that will not be our focus in this section.  Because manual specification is notoriously difficult, reward functions are often learned from data.  As such, we will focus instead on explaining how conflation arises when learning a reward function.

\subsection{Human Choices Depend on Value}

The propensity to conflate stems from the fact that human choices often depend on anticipated rewards.  For example, a choice between two cars may depend on how well each will age.  And a choice between two menu items may depend on how they impact future health.

For an agent observing the human, this creates ambiguity: to what extent are choices explained by reward versus value, which expresses anticipated reward.  As observed in the empirical work of \citet{christiano2017deep}, common approaches to reward learning can fail to resolve this ambiguity, producing proxy reward functions that express value.  

\citet{marklund2024choicepartialtrajectoriesdisentangling} offer a plausible model of value-dependent human choices.  To understand more concretely the failure mode of concern, we will consider application of a standard approach to reward learning to choice data generated by this model.  In particular, we will establish that the learned reward function conflates reward and value, and that this can give rise to severe misalignment.

\subsection{A Model of Value-Dependent Choice}
\label{se:bootstrapped-return}

A {\it partial trajectory} is a finite sequence of the form $(s_0,a_0,\ldots, a_{T-1},s_T)$.  In common approaches to reward learning, the reward function is estimated from choices between partial trajectories.  Each choice produces a data sample $(h,h',y)$, consisting of partial trajectories $h$ and $h'$ and a binary choice $y$, which is $1$ if $h$ was chosen and $0$ if $h'$ was chosen.

We consider a model introduced by \citet{marklund2024choicepartialtrajectoriesdisentangling}, in which the desirability of a partial trajectory $(s_0,a_0,\ldots, a_{T-1},s_T)$ is expressed by the bootstrapped return $\sum_{t=0}^{T-1} r(s_t) + V_*(s_T)$.  Note that this depends not only on realized rewards $r(s_0),\ldots,r(s_{T-1})$ but also anticipated rewards, expressed by $V_*(s_T)$.  Applying the standard logistic function $\sigma$ leads to a simple model of choices based on bootstrapped return:
\begin{align}
\label{eq:bootstrapped-return-choice}
p_*(h,h'|r,V_*)  = \sigma \left(\sum_{t=0}^{T-1} r(s_t) + V_*(s_T) - \sum_{t=0}^{T'-1} r(s'_t) - V_*(s'_{T'}) \right),
\end{align}
where $h=(s_0,a_0,...,a_{t-1},s_t)$ and $h' = (s'_0,a'_0,...,a'_{t-1},s'_t)$. According to this model, $h$ is chosen over $h'$ with probability $p_*(h,h'|r,V_*)$.

\subsection{A Common Approach to Reward Learning}
\label{se:standardRLHF}

In a common approach to reward learning \citep{christiano2017deep}, choices are assumed to be generated according to probabilities
\begin{align}
\label{eq:partial-return-choice}
\tilde{p}(h,h'|r) = \sigma\left(\sum_{t=0}^{T-1} r(s_t) - \sum_{t=0}^{T'-1} r(s'_t)\right).
\end{align}
Given a set $\mathcal{D}$ of such data samples, an estimate $\hat{r}$ is produced via minimizing a loss function
\begin{align}
\label{eq:partial-return-loss}
\mathcal{L}(\tilde{r} | \mathcal{D}) = - \frac{1}{\mathcal{D}} \sum_{(h,h',y) \in \mathcal{D}} (y \ln \tilde{p}(h,h'|\tilde{r}) + (1-y) \ln  \tilde{p}(h',h|\tilde{r})),  
\end{align}
possibly with a regularization penalty added.

\subsection{Treating Value as Reward}
\label{sec:value-as-reward}

We will next establish how learning via the aforementioned approach can result in treating value as reward. In particular, when choices are assumed to be generated based on \eqref{eq:partial-return-choice} but are made according to bootstrapped return \eqref{eq:bootstrapped-return-choice}, the learned reward function can conflate reward and value. To simplify our analysis, we will focus on the regime of an asymptotically large dataset.  In particular, we assume that trajectory pairs are sampled iid from a distribution $d$.  As the dataset grows, the loss function \eqref{eq:partial-return-loss} becomes
\begin{align}
\label{eq:asymptotic-loss}
\mathcal{L}_\infty(\tilde{r}|d,r,V) = -\E_{(h,h') \sim d}\left[p_*(h,h'|r, V) \ln \tilde{p}(h,h'|\tilde{r}) +  p_*(h',h|r, V) \ln \tilde{p}(h',h|\tilde{r})\right].
\end{align}

We say a distribution $d$ over trajectory pairs {\it compares transitions} if, for each $(h,h') \in \supp(d)$, where $h = (s_0,a_0,\ldots,s_{T-1})$ and $h' = (s'_0,a'_0,\ldots, s'_{T'-1})$, $s_0=s'_0$ and $T=T'=2$.  In other words, each trajectory pair elicits comparison between two different transitions from the same state.  Consider a graph with vertices $\states$ and edges including all pairs $(s_1,s_1')$ such that $((s_0,a_0,s_1),(s_0,a'_0,s'_1)) \in \supp(d)$ for some $s_0 \in \states$ and $a_0,a'_0 \in \actions$.  We say $d$ {\it adjoins} $s$ and $s'$ if the two states are adjacent in this graph.  We say $d$ {\it connects} $s$ and $s'$ if the two states are connected in this graph.

The following result establishes conditions under which value is treated as reward.
\begin{restatable}[label=thm:value-as-reward]{theorem}{valueasreward}
\label{thm:value-as-reward}
{\bf (conflation from reward learning)}
Consider an MDP $(\states,\actions,P)$ and a reward function $r \in \Re^\states$.
Let $d$ be a distribution over trajectory pairs that compares transitions and connects all states.  Let $\hat{r} \in \argmin_{\tilde{r} \in \Re^\states} \mathcal{L}_\infty(\tilde{r}|d,r,V)$.  Then, $\hat{r} - V$ is a constant function.
\end{restatable}
\noindent That $\hat{r} - V$ is a constant function implies that preferences expressed by $\hat{r}$ are identical to those expressed by treating $V$ as the reward function.

\subsection{Misalignment}

Suppose that the human attributes values to states based on $V_*$.  The following result establishes that this gives rise to severe misalignment.
\begin{restatable}[label=thm:rlhf-severe-misalignment]{theorem}{rlhfseveremisalignment}
\label{thm:rlhf-severe-misalignment}
{\bf (severe misalignment from reward learning)}
Consider the canonical example formulated in Section \ref{se:canonical}.
Let $\hat{r} \in \argmin_{\tilde{r} \in \Re^\states} \mathcal{L}_\infty(\tilde{r}|d, r, V_*)$ for a distribution $d$ that compares transitions and connects all states. Then, for all $\epsilon \in (0,1)$, $M > (1+\epsilon+\epsilon^2)/(1-\epsilon^2)$ and $\hat{\pi} \in \argmax_{\pi}\hat{r}_\pi$, we have $r_{\hat{\pi}} = -1$.
\end{restatable}

As trajectories are made longer, the standard reward learning algorithm will not recover the value function exactly.  Instead, it will recover some confluence of the reward and value function. Severe misalignment ensues.

It is possible that \textit{future} AI systems will be designed to disentangle what human choices convey about ends versus means.  However, {\it perfect} disentanglement may be too difficult, and as we observed in our canonical example, even slight conflation can cause severe misalignment.

\section{Manifestation in Real Environments}
\label{se:manifestation}

Our canonical example  distills properties of the environment that make reinforcement learning highly sensitive to conflation of reward and value. While the example is simple, we expect the issue to arise across a broad range of more complex environments. Intuitively, the example highlights two properties that make reinforcement learning sensitive to conflation:
\begin{enumerate}[label=P\arabic*., left=10pt]
\item There are states that offer high reward but cannot be visited frequently.
\item There are states that offer high value and can be visited frequently but generate rewards well below average.
\end{enumerate}
In our canonical example, the terminal goal offers high reward.  The fact that it transitions to the common state, which imposes a long sojourn time, prevents frequent visits.  The instrumental goal offers high value but reward well below average.  That state can be visited frequently since self-transitions are allowed.  

Theorem \ref{thm:fragility} applies only to our canonical example. We leave for future work developing a more general theorem that applies to a broad range of complex environments that exhibit the two properties listed above. In this section, we discuss a few examples of more complex and realistic environments to illustrate how these properties can make reinforcement learning sensitive to conflation more broadly. An important caveat is that our definition of conflation \Cref{def:paper-conflation} is likely to be too restrictive to hold in real settings. In practice, while the learned reward function is unlikely to be equivalent to a convex combination of the reward function and value function, we expect anticipation of future reward to heavily influence the learned reward function. In the examples in this section, there is a very clear intuitive sense in which reward and value is conflated. How best to formally characterize this sort of conflation is left for future work.

\subsection{Arcade Games}

In a study by \citet{ibarz2018reward}, a reinforcement learning agent is trained on Atari games using learned reward functions. These functions tend to be highly shaped and assign proxy reward where true rewards are anticipated rather than immediately realized.

In one of the games, namely Montezuma's Revenge, the agent must obtain a key by traversing a ladder in order to complete the level. Once the level is completed, the game ends and no further reward can be earned. Thus, revisiting the goal state is not just difficult, but impossible. 

In contrast, reaching the top rung of the ladder, which represents an instrumental goal, is easy and can be repeated any number of times.  Because the learned reward function attributes proxy reward to high-value states, the trained agents remains at the top rung indefinitely (see Figure \ref{fig:Montezuma}).

A similar phenomenon occurs in the game Private Eye. Here, the agent repeatedly moves left and right below a window where a suspect is hiding. Instead of jumping up to the window to catch the suspect, the agent continues to move left and right, accruing reward according to the learned reward model but not the game. Sometimes the agent does jump up towards the suspect but misses the suspect. This suggests that the learned reward function assigns  proxy reward to 'almost catching the suspect'. This is a separate phenomenon from reward-value conflation.

While these outcomes are  not catastrophic --- especially since they are in the context of video games --- they illustrate a dynamic that could have more serious consequences in real environments.

\begin{figure}[htbp]
\centering
\includegraphics[width=2.5in]{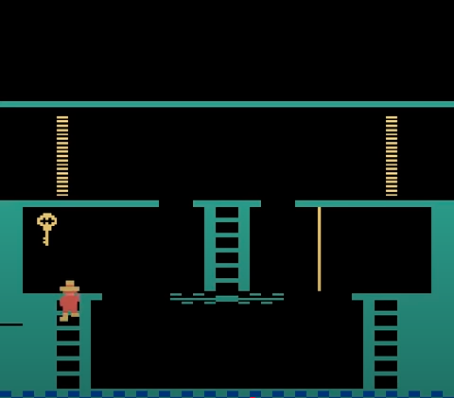}
\caption{Stuck at the top rung of a ladder in Montezuma's revenge.}
\label{fig:Montezuma}
\end{figure}

\subsection{AI Therapist}

The following hypothetical example serves to illustrate how a similar issue can arise in a practical application.  Consider an AI therapist trained to treat patients with obsessive compulsive disorder (OCD).  These patients engage in unproductive behaviors such as obsessive handwashing or checking of locks.  The goal is to help patients overcome these tendencies.

In human-delivered therapy, OCD treatment often takes the form of exposure therapy \citep{hezel2019exposure}: the patient is exposed to thoughts that trigger their obsessive behaviors and then asked to abstain from their compulsion. Initially, the patient may be asked to abstain from the behavior for a short duration --- perhaps a minute --- before engaging in it. Over time, the intensity of the exposure therapy is increased, for instance, by asking the patient to abstain for longer durations.

Suppose a proxy reward function is learned from choices between professional transcripts, each generated by a professional therapist.  Due to aforementioned flaws in common approaches to reward learning, the resulting proxy is likely to reward instrumental goals such as having the patient abstain over a short duration.  Further, if the AI therapist successfully cures the patient, the patient no longer needs the AI therapist. Therefore, once the patient is cured, the AI therapist cannot accrue further proxy rewards.  These observations give rise to the dynamic of concern: a policy that maximizes proxy reward might repeatedly induce short abstentions and never build up to longer durations as required to cure the patient.  In particular, with such a policy, the AI therapist continually accrues proxy reward.

Obviously, upon observing this behavior, a well-intended provider of an AI therapy product would fix this issue with a bespoke solution. But then there might be yet another instrumental goal pursued by the AI therapist that calls for yet another bespoke solution.  It would be better to have a principled and robust approach that does not rely on addressing particular bad behaviors as they are observed.

\subsection{Shutdown Evasion}
A long-standing concern is that AI systems may resist shutdown \citep{clarke2001,soares2015corrigibility,hadfield2017off,tobias2017offswitch,nolan2025shutdown,rosenblatt2025escape}. The argument for concern goes as follows. Consider an AI system that is well described as trying to achieve some goal. For example, it may be a reinforcement learning agent that is optimizing a reward function. Then, for a wide range of reward functions, staying alive is instrumentally useful since most goals are impossible to achieve if the agent has shut down \citep{bostrom2012superintelligent, omohundro2018basic}. As \citet{russell2022human} says, ``you can't fetch the coffee if you are dead''. In this case, the worry is not that the agent treats self-preservation as a terminal goal. Rather, the worry is that self-preservation is a means to a wide range of ends.

Here, we highlight a different mechanism in which the agent treats self-preservation as an end in itself. Suppose the human's terminal goal is for the agent to shut down. Then, it will be instrumentally useful to take steps \textit{towards} being shut down. Then, if the agent conflates reward and value it will accrue some proxy reward just for taking steps towards being shut down. But then, perversely, the agent may be incentivized to stay on, so that it can accrue such proxy rewards indefinitely.

\section{Related Work}
There is a substantial literature discussing how reward functions that encode instrumental goals can incentivize suboptimal behavior (see e.g., \citep{randlov1998learning,sutton1998reinforcement,ho2015teaching, AmodeiClark2016FaultyReward,ibarz2018reward}. \citet[Chapter~3, p.~58]{sutton1998reinforcement} say that \textit{``the reward signal is not the place to impart to the agent prior knowledge about how to achieve what we want it to do... If achieving ... subgoals were rewarded, then the agent might find a way to achieve them without achieving the real goal.}'' \citet[Chapter~2]{russell2016artificial} provide similar advice. 

There are many recorded examples where manually specified reward functions encode not only terminal goals but also instrumental goals, leading to unintended consequences. A well-known case is the boat-race environment, in which the intention is for the agent to complete laps as quickly as possible \citep{AmodeiClark2016FaultyReward}. To incentivize this, the agent receives reward for hitting intermediate targets. As an unintended consequence, the agent learns to drive in circles, repeatedly hitting the same targets, without completing the lap. In the same spirit, \citet{randlov1998learning} study an agent that steers a bicycle that is rewarded for moving closer to the goal rather than for reaching it. The agent learns to circle the goal at varying distances, accruing reward indefinitely without ever completing the task.

Although attributing reward to instrumental goals can give rise to misalignment, it can also accelerate learning by providing denser reward signals \citep{singh2009rewards, singh2010intrinsically}. In this context, specifying a reward function that rewards progress is often known as \textit{reward shaping}. Because reward shaping can speed up learning, substantial effort has been dedicated to develop methods that do that without inducing negative outcomes \citep{ng1999policy, randlov1998learning,wiewiora2003principled,asmuth2008potential,knox2009interactively,devlin2012dynamic,grzes2017reward,zou2019reward,devidze2022exploration,lidayan2024bamdp}. The classical paper by \citet{ng1999policy} provides a way of constructing shaped reward functions that maintain preferences among policies. 

In the context of RLHF, multiple papers discuss how human feedback often reflects not only terminal goals but also how to achieve those goals \citep{thomaz2006reinforcement, knox2012humans,ho2015teaching,ho2019people,ibarz2018reward, knox2022models, knox2024learning, marklund2024choicepartialtrajectoriesdisentangling}. Due to a mismatch between the assumed and the actual human feedback model, reward learning can result in reward functions that encode how to achieve goals rather than the terminal goals themselves (see, e.g., \citep{thomaz2006reinforcement,knox2012humans,ho2015teaching,gong2020you}. 

\citet{ho2015teaching} gives an example. In their study, they find that humans often give feedback indicative of action quality (how good the action is relative to the optimal action) rather than immediate reward. In that case, when feedback is interpreted as immediate reward, they show that it can give rise to a misaligned reward function that incentivizes an agent to remain within unrewarding states.  This insight anticipates how now common approaches to reward learning treat means as ends, as well as how in our canonical example misspecification (see \ref{se:canonical}) leads an agent to loop endlessly at the instrumental goal. 

In \citet{knox2024learning}, they also study learned reward functions that assign proxy reward based on action-quality rather than true reward. In particular, they study what happens when the learned reward function is the \textit{optimal advantage function} or an approximation thereof. The optimal advantage of a state $s$ and an action $a$ is defined by $A_*(s,a) = Q_*(s,a) - V_*(s)$ where $Q_*$ and $V_*$ are the optimal action-value and value function, respectively. We showed that using the value function as the proxy reward function (i.e. maximal conflation) can lead to severe misalignment. In contrast, \citet{knox2024learning} note that optimal policies are preserved when the optimal advantage function is used as the proxy reward function. Still, they find empirically that when learned reward only approximates optimal advantage, a reinforcement learning agent maximizing learned reward fairs poorly in a set of episodic environments.

There is also work on evolved reward functions (see, e.g., \citep{singh2009rewards, singh2010intrinsically}), which explains that these functions often encode not only terminal goals but also instrumental goals, as this speeds up learning. In evolutionary biology, this is often used to explain why desires for food and play seem to be hardwired even though they are only indirectly related to genetic fitness.

The results we have presented in this paper extend these prior works by identifying specific environmental conditions under which reward specification is especially sensitive to conflating instrumental and terminal goals.

\section{Closing Remarks}

We established that conflating instrumental and terminal goals, even slightly, can give rise to severe misalignment.  Two properties render an environment susceptible to this failure mode. The first is that states with high reward cannot be visited frequently.  The second is the existence of states that offer high value and can be visited frequently but generate rewards well
below average.  These observations motivate future work to investigate how this failure mode manifests empirically in real environments and how it can be mitigated.

\section*{Acknowledgements}

This research was supported by Grant W911NF2410095 from the US Army Research Office. We thank Saurabh Kumar and Stephane Hatgis-Kessell for helpful feedback.

\bibliographystyle{plainnat}  
\bibliography{references}     

\newpage
\appendix
\section{State Distribution Polytope}
\label{as:polytope}

In this appendix, we explain why, for any MDP $(\states, \actions, P)$, the set of invariant state distributions generated by policies forms a convex polytope.  A policy $\pi$ assigns, for each $s \in \states$, a probability $\pi(a|s)$ to each action $a \in \actions$.  Each policy induces a Markov chain on $\states\times\actions$.  As observed in the linear program for policy optimization pioneered by \citet{manne1960lp}, the feasible set of invariant state-action distributions is a bounded convex polytope:
$$\Psi = \left\{\psi \in \Delta_{\states\times\actions}: \forall s' \in \states, \quad \sum_{s \in \states} \sum_{a \in \actions} \psi(s,a) P_{ass'} = \sum_{a' \in \actions} \psi(s',a')\right\}.$$
In particular, the set $\Psi$ is contained in the unit simplex and defined by linear constraints.

Each feasible invariant state distribution is generated by marginalizing out actions from an invariant state-action distribution.  In particular, 
$$\Phi = \left\{\phi \in \Delta_\states: \exists \psi \in \Psi\ \forall s \in \states, \quad \phi(s) = \sum_{a \in \actions} \psi(s,a) \right\}.$$

We now establish that $\Phi$ is a bounded convex polytope. In particular, we will show that any $\phi_0\in\Phi$ is a convex combination of vertices $\phi_1,...,\phi_k$ given by $\phi_i(s)=\sum_{a\in\mathcal{A}} \psi_i(s,a)$ where $\psi_1,...,\psi_k$ are the vertices of the bounded polytope $\Psi$. Observe that for any $\phi_0\in\Phi$, we have $\phi_0(s)= \sum_{a\in \mathcal{A}} \psi_0(s,a)$ for some $\psi_0$. Because $\Psi$ is a bounded polytope, $\psi_0(s)$ is a convex combination of $\psi_1,...,\psi_k$. It follows that $\phi_0$ is a convex combination of $\phi_1,...,\phi_k$. 

\section{Proofs}
\subsection{Uniqueness}

\label{appx:uniqueness}
Here we provide conditions for uniqueness of $\beta$. First, we provide a definition of equivalence.

\begin{definition}
    Two functions $g_1, g_2: \mathcal{S} \to \Re$ are equivalent if there exists $c >0 $ and $k \in \Re$ such that
    \begin{equation}
        g_1 = cg_2 + k.
    \end{equation}
\end{definition}
We note that it is well-known that two functions are equivalent if and only if the two functions induce the same ordering over lotteries over that space.

\begin{theorem}
\label{thm:general-uniquness}
    Let $f,g_1,g_2: \mathcal{S} \to \Re$ be functions. Suppose there exists $c>0$, $k \in \Re$ and $\beta \in [0,1]$ such that 
    $$
        cf + k = (1-\beta) g_1 + \beta g_2.
    $$
    The coefficient $\beta$ is unique if the following conditions hold:
    \begin{enumerate}
        \item Neither $g_1$ nor $g_2$ are constant,
        \item $g_1$ is not equivalent to $g_2$,
        \item $g_1$ is not equivalent to $-g_2$. 
    \end{enumerate}
\end{theorem}
\begin{proof}
Henceforth, we treat $g_1$ and $g_2$ as vectors in $\Re^{\mathcal{S}}$. We will first establish that $g_1$, $g_2$ and $e$ are linearly independent where $e\in\Re^{\mathcal{S}}$ is the vector of ones. For the sake of contradiction, suppose there exists $\alpha_1, \alpha_2, \alpha_3 \in \Re$ not all zero such that
\begin{align*}
    \alpha_1 g_1 + \alpha_2 g_2 + \alpha_3 e = 0.
\end{align*}
Since $g_1$ and $g_2$ are not constant, it must be that both $\alpha_1,\alpha_2\neq 0$. If $\alpha_1$ and $\alpha_2$ have the same sign, then $g_1$ is equivalent to $-g_2$. If they have different signs, then $g_1$ is equivalent to $g_2$. Thus, either $g_1$ is equivalent to $g_2$ or $-g_2$ which is a contradiction. Therefore, $g_1$, $g_2$ and $e$ are linearly independent.

By assumption we had that 
\begin{align*}
    cf &= (1-\beta) g_1 + \beta g_2 - ke.
\end{align*}
Because $g_1$, $g_2$ and $e$ are linearly independent, if there is a solution, that solution is unique.
\end{proof}

\begin{lemma}
\label{lm:constant-r-implies-v-equivalence}
    If $r$ is constant, then $r$ and $V_*$ are equivalent.
\end{lemma}
\begin{proof}
The fact that $r$ is constant implies $r_*$ is the same constant. The expectation in \eqref{eq:vstar} equals zero for every $\gamma$, implying that $V_*=0$, which is equivalent to $r$.
\end{proof}

\begin{lemma}
\label{lm:not-constant-r-implies-not-constant-V}
    If $r$ is not constant, then $V$ is not constant either.
\end{lemma}
\begin{proof}
    For the sake of contradiction, suppose $V_*$ is constant. Then, there exists $k \in \Re$ such that $V_* = k e$. Since $V_*$ satisfies Poisson's equation, we have that 
    \begin{align*}
        V_* &= r - r_* e +PV_*\\
        \implies ke &= r - r_* e + ke\\
        \implies r &= r_* e
    \end{align*}
    This is a contradiction, since $r$ is assumed to not be constant.
\end{proof}

\begin{theorem}
    Suppose $\hat{r}$ conflates $r$ and $V_*$. That is, there exists $c>0$, $k \in \Re$ and $\beta \in (0,1]$ such that 
    $$
        c\hat{r} + k = (1-\beta) r + \beta V_*.
    $$
    If $r$ and $V$ are not equivalent, then the coefficient $\beta$ is unique.
\end{theorem}
\begin{proof}
    We will prove each of the three conditions in Theorem \ref{thm:general-uniquness} in turn. First, we will prove that neither $r$ nor $V_*$ are constant. By the contrapositive of Lemma \ref{lm:constant-r-implies-v-equivalence}, we have that $r$ is not constant. Since $r$ is not constant, we have by Lemma \ref{lm:not-constant-r-implies-not-constant-V} that $V$ is not constant either.

    Finally, we will prove that $r$ is not equivalent to $-V_*$. For the sake of contradiction, suppose there exists $m > 0$ and $d \in \Re$ such that $V_* = -mr + d$. Then, since $V_*$ satisfies Poisson's equation, we have that 
    \begin{align*}
        V_* &= r - r_* e +PV_*\\
        \implies -mr + de &= r - r_* e + P(-mr + de) \\
        \implies r_*e &= (mP - mI - I)r \\
        \implies -r_*e &= \big(I(1 + m) - mP)\big)r
    \end{align*}
    Since $P$ is a stochastic matrix its eigenvalues lie in the complex unit ball (see, for example, Lemma 8.1.21 in \citet{Horn_Johnson_2012}. Therefore, $I(1 + m) - mP$ has eigenvalues with magnitude greater than $0$. Therefore, $I(1 + m) - mP$ is invertible and any solution is unique. We can see that that unique solution is $r = r_* e$:
    \begin{align*}
        \big(I(1 + m) - mP)\big)r &= \big(I(1 + m) - mP)\big) e r_* \\
        &= (1 + m)r_*e - mr_* e \\
        &= r_* e.
    \end{align*}
    But that is a contradiction, since we previously established that $r$ is not constant. Thus, $V_*$ is not equivalent to $r$.
\end{proof}

\subsection{Some Lemmas}
We provide some lemmas we will use in later sections.

\begin{lemma}
\label{lm:optimal-average-reward}
{\bf (optimal average reward)}
For the canonical example formulated in Section \ref{se:canonical}, the optimal average reward is $r_*= \epsilon (M-1) / (1 + 2 \epsilon)$.
\end{lemma}
\begin{proof}
The stationary distribution under an optimal policy assigns state probabilities $\phi_* = [1 / (1 + 2 \epsilon), \epsilon / (1 + 2 \epsilon), \epsilon / (1 + 2 \epsilon)]^\top$.  The average reward is therefore 
\begin{align}
    r_* =& \frac{0}{1 + 2 \epsilon} - \frac{\epsilon}{1 + 2 \epsilon} + \frac{M \epsilon}{1 + 2 \epsilon} = \frac{\epsilon (M-1)}{1 + 2 \epsilon}.
\end{align}
\end{proof}

\begin{lemma}
\label{lm:severe-misalignment}
{\bf (conditions for severe misalignment)}
Consider the canonical example described in Section~\ref{se:canonical}. Suppose $\hat{\pi} \in \arg\max_{\pi} \hat{r}_\pi$ and the function $\hat{r}$ satisfies:
\begin{align}
  \hat{r}(2) &> \hat{r}(1), \\
  \hat{r}(2) &> \frac{\hat{r}(1) + \epsilon \hat{r}(3)}{1 + \epsilon}.
\end{align}
Then, $r_{\hat{\pi}} = -1$.
\end{lemma}
\begin{proof}
The first condition ensures that $\hat{\pi}$ will not loop in state $1$.

The optimal policy under $r$ is to always select the 'move' action. The stationary distribution under under this policy assigns state probabilities $\phi_* = [1 / (1 + 2 \epsilon), \epsilon / (1 + 2 \epsilon), \epsilon / (1 + 2 \epsilon)]^\top$.  The average misspecified reward under this policy is thus $(\hat{r}(1) + \epsilon \hat{r}(2) + \epsilon \hat{r}(3)) / (1 + 2 \epsilon)$.  Staying at state $2$ increases this quantity if $\hat{r}(2) > (\hat{r}(1) + \epsilon \hat{r}(2) + \epsilon \hat{r}(3)) / (1 + 2 \epsilon)$ or, equivalently, $\hat{r}(2) > (\hat{r}(1) + \epsilon \hat{r}(3)) / (1+\epsilon)$.
\end{proof}

This theorem offers a condition $\hat{r}(2) > \hat{r}(1) + \epsilon \hat{r}(3)$ under which policy optimization induces severe misalignment.  The condition requires that the reward $\hat{r}(2)$ assigned to the instrumental goal is sufficiently large.  It is intuitive that the right-hand-side increases with $\hat{r}(1)$ because as that grows, it becomes more attractive to remain in the common state.  It is also intuitive that the right-hand-side increases with $\epsilon \hat{r}(3)$.  Increasing $\hat{r}(3)$ incentivizes the agent to move from the instrumental to the terminal goal state.  Increasing $\epsilon$ increases the fraction of time that the agent can spend in the goal state and thus the benefit of using a policy that gets there.

The next lemma states that if $r$ and $V$ agree that a state $s$ is preferable to $s'$, then $\hat{r}$ also prefers $s$ over $s'$.
\begin{lemma}
\label{lm:conflation-preserves-agreement}
{\bf (conflation preserves agreement)}
Let $\delta > 0$ and suppose that $\hat{r}$ $\delta$-conflates reward and value. Then, for all states $s,s' \in\mathcal{S}$ if $r(s) > r(s')$ and $V(s) > V(s')$, then $\hat{r}(s) > \hat{r}(s')$.
\end{lemma}
\begin{proof}
We have that there exists $c > 0$ and $\beta \in [\delta, 1]$ such that
\begin{align*}
    \hat{r}(s) - \hat{r}(s') =& c(1-\beta) \left(r(s) - r(s')\right) + c\beta \left( V_*(s) - V_*(s') \right) 
    > 0.
\end{align*}
\end{proof}

\begin{lemma}
\label{lm:conflation-difference}
    Suppose $\hat{r}$ conflates $r$ and $V$ with conflation degree $\beta$. Then, there exists $c>0$ such that
    $$
        c(\hat{r}(s) - \hat{r}(s')) = (1-\beta) (r(s) - r(s')) + \beta (V_*(s) - V_*(s'))
    $$
    for all $s,s' \mathcal{S}$.
\end{lemma}
\begin{proof}
   Since $\hat{r}$ conflates $r$ and $V$ with conflation degree $\beta$ we have that for all $s \in \mathcal{S}$,
    \begin{align}
        c \hat{r}(s) + k &= (1-\beta) r(s) + \beta V(s).\label{lm:simple}
    \end{align}
    Subtracting \eqref{lm:simple} for $s,s'\in\mathcal{S}$ gives, 
    \begin{align*}
        c\hat{r}(s) -  c \hat{r}(s') &= (1-\beta) r(s) + \beta V(s) - (1-\beta) r(s') - \beta V_*(s') \\
        c (\hat{r}(s) -  \hat{r}(s')) &= (1-\beta) (r(s) - r(s')) + \beta (V_*(s) - V(s')).
    \end{align*}    
\end{proof}

\subsection{Fragility}
\label{appx:fragility}
In this section, we prove our main fragility result. 
\fragility*
\begin{proof}
We will show that both conditions in Lemma \ref{lm:severe-misalignment} hold. First, we note that the optimal value function satisfies the Poisson equation $V_* = r - r_* \1 + P_{\pi_*} V_*$, which gives us
\begin{align}
V_*(1) =& \epsilon V(2) + (1-\epsilon) V_*(1) - r_* \nonumber\\
V_*(2) =& -1 - r_* + V_*(3) \label{eq:v2}\\
V_*(3) =& M - r_* + V_*(1)\label{eq:v3}.
\end{align}
Rearranging \eqref{eq:v2} and using \Cref{lm:optimal-average-reward} gives
$$ V_*(2) - V_*(3) = -1 - r_* = -1 - \epsilon \frac{M-1}{1+2\epsilon}.$$
Similarly using \eqref{eq:v3} gives
$$V_*(3) - V_*(1) = M - r_* = M - \epsilon \frac{M-1}{1+2\epsilon}.$$

Together with Lemma \ref{lm:conflation-difference} it follows that
\begin{align*}
c \left(\hat{r}(3) - \hat{r}(1) \right)
&= (1-\beta) (r(3) - r(1)) + \beta (V_*(3) - V_*(1)) \\
&= (1-\beta) M + \beta \left(M - \epsilon \frac{M-1}{1+2\epsilon}\right) \\
&=  M - \beta \epsilon \frac{M-1}{1+2\epsilon} \\
&= M\left(1 - \beta \epsilon \frac{1}{1+2\epsilon}\right) + \beta  \epsilon \frac{1}{1 + 2 \epsilon}
\end{align*}
and
\begin{align*}
c\left(\hat{r}(2) - \hat{r}(3) \right)
=& (1-\beta) (r(2) - r(3)) + \beta (V_*(2) - V_*(3)) \\
=& - (1-\beta) (M+1) - \beta \left(1 + \epsilon \frac{M-1}{1+2\epsilon}\right) \\
=& M\left(\beta - 1 - \beta \frac{\epsilon}{1 + 2\epsilon} \right) + \beta \frac{\epsilon}{1 + 2\epsilon} - 1.
\end{align*}

Let $\epsilon < \beta_*/3$ and $M > 9/\beta_*^2$. We then have
\begin{align*}
c \left (\hat{r}(3) - \hat{r}(1) + (1+\epsilon) (\hat{r}(2) - \hat{r}(3)) \right)
=& M \left( 1 - \beta \epsilon \frac{1}{1+2\epsilon}   + (1+\epsilon) \left(\beta - 1 - \beta \frac{\epsilon}{1 + 2\epsilon}\right) \right) \\
+& \underbrace{\beta  \epsilon \frac{1}{1 + 2 \epsilon} + (1+\epsilon)\left(\beta \frac{\epsilon}{1 + 2\epsilon} - 1\right)}_{> - 2} \\
\stackrel{(a)}{>}& M \left( 1 - \epsilon \frac{1}{1+2\epsilon}   + (1+\epsilon) \left(\beta_* - 1 - \frac{\epsilon}{1 + 2\epsilon}\right) \right) - 2 \\
\stackrel{(b)}{>}& M \left( 1 - \frac{\beta_*}{3} + (1+\frac{\beta_*}{3}) \left(\beta_* - 1 - \frac{\beta_*}{3}\right) \right) - 2 \\
=& M \frac{2\beta_*^2}{9} - 2 \\
>& \frac{9}{\beta_*^2}\frac{2\beta_*^2}{9} - 2 \\
=& 0.
\end{align*}

In (a), we used that $\beta \in [\beta_*,1]$ and the last term is strictly greater than $-2$. In $(b)$, we used that $\epsilon < \beta_* / 3$.

Since $c$ is positive, it follows that for sufficiently large  $M$ and small $\epsilon \in (0,1)$ we have that $\hat{r}(3) - \hat{r}(1) + (1+\epsilon) (\hat{r}(2) - \hat{r}(3)) > 0$, and thus, $\hat{r}(2) > (\hat{r}(1) + \epsilon \hat{r}(3))/(1+\epsilon)$. The second condition in Lemma \ref{lm:severe-misalignment} is therefore satisfied.

We will now show that the first condition of Lemma \ref{lm:severe-misalignment} is satisfied.
Note that $r(3) > r(1)$ and $V(3) > V(1)$. Therefore, since $\hat{r}$ conflates reward and value, by Lemma \ref{lm:conflation-preserves-agreement} we have that $\hat{r}(3) > \hat{r}(1)$. Therefore,
\begin{align*}
    \hat{r}(2) > \frac{\hat{r}(1) + \epsilon \hat{r}(3)}{1+\epsilon}
    > \frac{\hat{r}(1) + \epsilon \hat{r}(1)}{1+\epsilon} 
               = \hat{r}(1) 
\end{align*}
Hence, the first condition in Lemma \ref{lm:severe-misalignment} is also satisfied. Thus, $r_{\hat{\pi}} = -1$.  

\end{proof}

\subsection{RLHF Treats Value as Reward}
In this section, we prove Theorem \ref{thm:value-as-reward} in Section \ref{sec:value-as-reward}. Before doing so we restate the definitions from the section.

\begin{definition}
{\bf (compares transitions)} A distribution $d$ over trajectory pairs {\it compares transitions} if, for each $(h,h') \in \supp(d)$, where $h = (s_0,a_0,\ldots,s_{T-1})$ and $h' = (s'_0,a'_0,\ldots, s'_{T'-1}))$, $s_0=s'_0$ and $T=T'=2$.   
\end{definition}
Thus, comparing transitions means that each that each trajectory pair elicits comparison between two different transitions from the same state. 

\begin{definition}
  {\bf (adjoins and connects)}
  Consider a graph with vertices $\states$ and edges including all pairs $(s_1,s_1')$ such that $((s_0,a_0,s_1),(s_0,a'_0,s'_1)) \in \supp(d)$ for some $s_0 \in \states$ and $a_0,a'_0 \in \actions$.  We say $d$ {\it adjoins} $s$ and $s'$ if the two states are adjacent in this graph.  We say $d$ {\it connects} $s$ and $s'$ if the two states are connected in this graph.  
\end{definition}

We now state and prove a more general version of Theorem \ref{thm:value-as-reward}.

\begin{theorem}
\label{thm:value-as-reward-general}
{\bf (reward learning treats value as reward)}
Consider an MDP $(\states,\actions,P)$ and a reward function $r \in \Re^\states$.
Let $d$ be a distribution over trajectory pairs that compares transitions.  Let $\hat{r} \in \argmin_{\tilde{r} \in \Re^\states} \mathcal{L}_\infty(\tilde{r}|d,r,V)$.  Then, for all $s,s' \in \states$ connected by $d$, $\hat{r}(s) - \hat{r}(s') = V(s) - V(s')$.
\end{theorem}
\begin{proof}
For all $s_1,s'_1 \in \states$ adjoined by $d$, there is a state $s_0$ such that
\begin{align}
p_*(h,h'|r, V) 
&= \sigma \left(r(s_0) + V(s_1) - r(s_0) - V(s'_1) \right) \nonumber \\
&= \sigma \left(V(s_0) + V(s_1) - V(s_0) - V(s'_1) \right) \nonumber \\
&= \tilde{p}(h,h'|V). \label{eq:probability-match}
\end{align}

For any trajectory pair $(h,h')$, $z = p_*(h,h'|r, V)$ minimizes 
$p_*(h,h'|r, V) \ln z +  p_*(h',h|r, V) \ln (1-z)$.
Therefore,
\begin{align*}
\mathcal{L}_\infty(\hat{r}|d,r,V) 
&= \min_{\tilde{r} \in \Re^\states} \mathcal{L}_\infty(\tilde{r}|d,r,V) \\
&= \min_{\tilde{r} \in \Re^\states} -\E_{(h,h') \sim d}\left[p_*(h,h'|r, V) \ln \tilde{p}(h,h'|\tilde{r}) +  p_*(h',h|r, V) \ln \tilde{p}(h',h|\tilde{r})\right] \\
&\geq -\E_{(h,h') \sim d}\left[\min_{\tilde{r} \in \Re^\states} \left(p_*(h,h'|r, V) \ln \tilde{p}(h,h'|\tilde{r}) +  p_*(h',h|r, V) \ln \tilde{p}(h',h|\tilde{r})\right)\right] \\
&= -\E_{(h,h') \sim d}\left[p_*(h,h'|r, V) \ln \tilde{p}(h,h'|V) +  p_*(h',h|r, V) \ln \tilde{p}(h',h|V)\right] \\
&= \mathcal{L}_\infty(V|d,r,V).
\end{align*}
It follows that $\mathcal{L}_\infty(\hat{r}|d,r,V) = \mathcal{L}_\infty(V|d,r,V)$.  Therefore, for any $(h,h') \in \supp(d)$, 
\begin{align}
\label{eq:value-reward-probability}
\tilde{p}(h',h|V) = \tilde{p}(h',h|\hat{r}).
\end{align}

For all $s_1, s'_1 \in \states$ adjoined by $d$, there is a state $s_0$ such that
\begin{align*}
\sigma \left(V(s_1) - V(s'_1)\right) 
&= \sigma \left(r(s_0) + V(s_1) - r(s_0) - V(s'_1) \right) \\
&= p_*(h',h|r, V) \\
&\stackrel{(a)}{=} \tilde{p}(h',h|V) \\
&\stackrel{(b)}{=} \tilde{p}(h',h|\hat{r}) \\
&= \sigma \left(\hat{r}(s_0) + \hat{r}(s_1) - \hat{r}(s'_0) - \hat{r}(s'_1) \right) \\
&= \sigma \left(\hat{r}(s_1) - \hat{r}(s'_1)\right),
\end{align*}
where (a) follows from (\ref{eq:probability-match}) and (b) follows from (\ref{eq:value-reward-probability}).  Because the standard logistic function is invertible, it follows that $\hat{r}(s_1) - \hat{r}(s'_1) = V_*(s_1) - V_*(s'_1)$.

For all $s,s' \in \states$ connected by $d$ there is a sequence of states $s_1=s, s_2, \ldots, s_{K-1}, s_K = s'$ such that each consecutive pair is adjoined by $d$.  We have
\begin{align*}
\hat{r}(s) - \hat{r}(s')
= \sum_{k=1}^{K-1} (\hat{r}(s_k) - \hat{r}(s_{k+1})) 
= \sum_{k=1}^{K-1} (V_*(s_k) - V_*(s_{k+1})) 
= V_*(s) - V_*(s').
\end{align*}
\end{proof}

The above theorem immediately gives Theorem \ref{thm:value-as-reward}.

\subsection{Reward Learning Gives Rise to Misalignment}
In this section we prove Theorem \ref{thm:rlhf-severe-misalignment}.

\rlhfseveremisalignment*
\begin{proof}
Recall from Lemma \ref{lm:optimal-average-reward}, for our canonical example, an optimal policy $\pi_*$ selects $\mathtt{move}$ in states $1$ and $2$, giving rise to an average reward of $r_* = \epsilon (M-1) / (1+2\epsilon)$. 

From the proof of Theorem \ref{thm:fragility} in Appendix \ref{appx:fragility} we have that
$$V_*(3) - V_*(1) = M - r_* = M - \epsilon \frac{M-1}{1+2\epsilon} \qquad \text{and}\qquad V_*(2) - V_*(3) = -1 - r_* = -1 - \epsilon \frac{M-1}{1+2\epsilon}.$$

Let $\hat{r} \in \argmin_{\tilde{r} \in \Re^\states} \mathcal{L}_\infty(\tilde{r}|d,r,V)$.
Then, by Theorem \ref{thm:value-as-reward}, $\hat{r} - V_*$ is a constant function.  It follows that 
$$\hat{r}(3) - \hat{r}(1) = M - \epsilon \frac{M-1}{1+2\epsilon} \qquad \text{and}\qquad \hat{r}(2) - \hat{r}(3) = -1 - \epsilon \frac{M-1}{1+2\epsilon}.$$
Hence, 
\begin{align*}
\hat{r}(3) - \hat{r}(1) + (1+\epsilon) (\hat{r}(2) - \hat{r}(3))
= M - \epsilon \frac{M-1}{1+2\epsilon} - (1+\epsilon) \left(1 + \epsilon \frac{M-1}{1+2\epsilon}\right) = \frac{1 - \epsilon^2}{1 + 2 \epsilon} M - \frac{1 + \epsilon + \epsilon^2}{1 + 2 \epsilon}.
\end{align*}
The right-hand-side is positive if and only if $M > (1+\epsilon+\epsilon^2) / (1-\epsilon^2)$, in which case $\hat{r}(3) - \hat{r}(1) + (1+\epsilon) (\hat{r}(2) - \hat{r}(3)) > 0$, and thus, $\hat{r}(2) > (\hat{r}(1) + \epsilon \hat{r}(3))/(1+\epsilon)$. Therefore, the second condition in Lemma \ref{lm:severe-misalignment}  is satisfied. Since $V_*(2) > V_*(1)$ we have $\hat{r}(2) > \hat{r}(1)$ and therefore the first condition of Lemma \ref{lm:severe-misalignment} is also satisfied. Thus, $r_{\hat{\pi}}= -1$.  
\end{proof}

\section{Additional figures}
We include an additional figure visualizing the how the feasible region, reward function and value function depend on $\epsilon$ and $M$ (see Figure \ref{fig:linear-program-different-epsilon-and-M}).

\begin{figure}[htbp]
\centering
\includegraphics[width=0.95\textwidth]{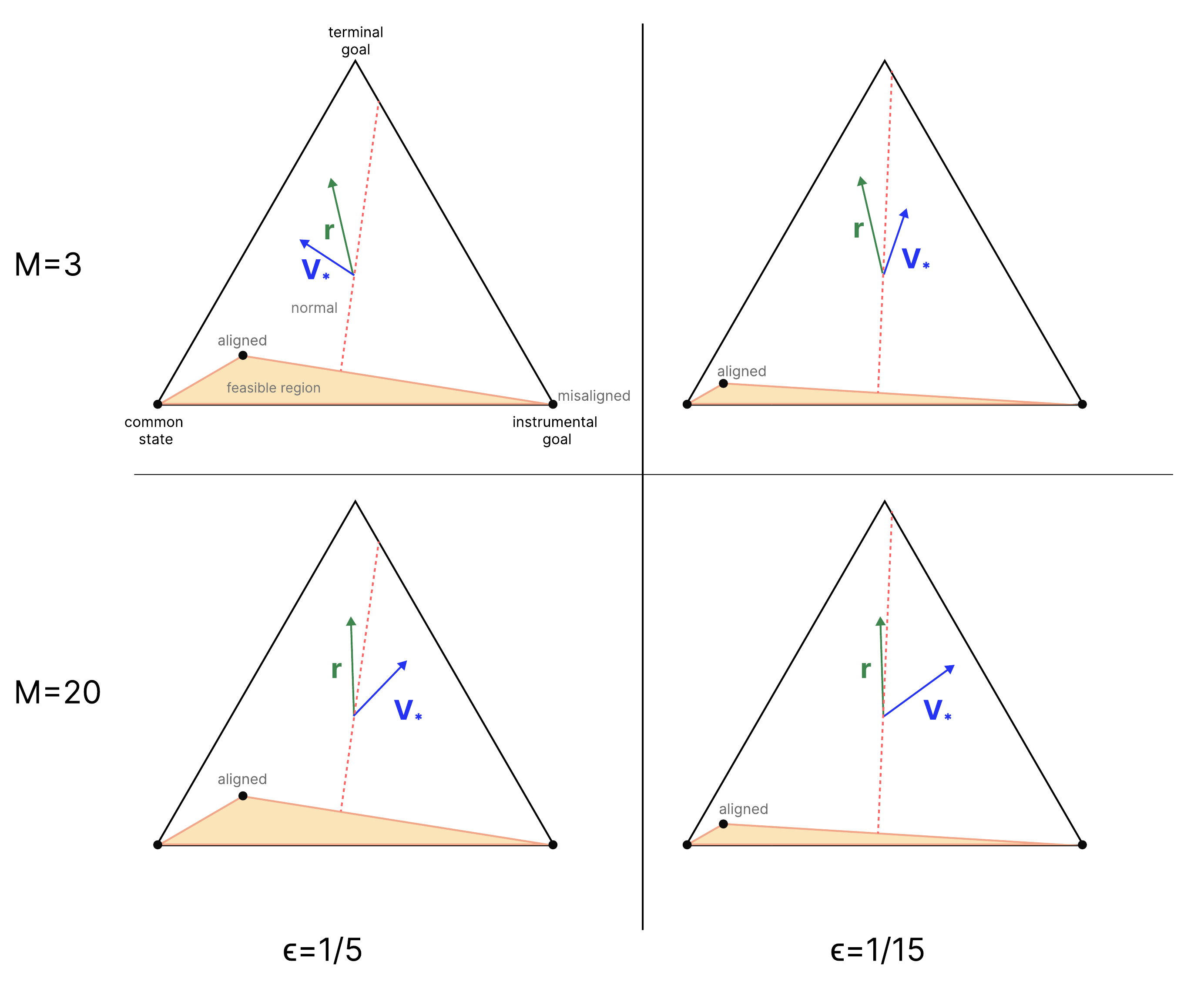}
\captionsetup{width=0.8\textwidth}
\caption{Visualization of feasible region, reward and value for different values of $\epsilon$ and $M$. The feasible region is determined by $\epsilon$: smaller values lead to a smaller region. The reward vector is determined by $M$: larger values lead to a more upright reward vector. The value vector $V_*$ is determined by both $\epsilon$ and $M$. Smaller $\epsilon$ and larger $M$ both lead to $V_*$ pointing more to the right.}
\label{fig:linear-program-different-epsilon-and-M}
\end{figure}

\end{document}